\documentclass{ecai}
\usepackage{times}
\usepackage{graphicx}
\usepackage{latexsym}

\ecaisubmission   

\newcommand{\vect}[1]{\overrightarrow{#1}}

\usepackage[utf8]{inputenc}
\usepackage[T1]{fontenc}
\usepackage{xspace}
\usepackage{amsmath,amssymb,amsthm}

\usepackage{booktabs}

\usepackage{tikz}
\usetikzlibrary{calc}
\usetikzlibrary{shapes}

\newcommand{\ab}[1]{{\color{orange}[AB] #1}}
\newcommand{\nm}[1]{{\color{blue}[NM] #1}}
\newcommand{\ps}[1]{{\color{purple}[PS] #1}}
\newcommand{\sylvain}[1]{{\color{cyan}[SB] #1}}

\newcommand{\Def}{\buildrel\hbox{\tiny \textit{def}}\over =}

\usepackage{stmaryrd}

\newcommand{\llb}{\llbracket}                               
\newcommand{\rrb}{\rrbracket}                               

\newcommand{\var}[1]{\mathbf{#1}}

\makeatletter
\newcommand{\hidden}[1]{\ifhmode \@bsphack \@esphack \fi}
\makeatother

\usepackage{booktabs}
\usepackage{tikz}
\usepackage{multirow}
\usepackage[linesnumbered,ruled,vlined]{algorithm2e}
\usepackage{balance}

\newtheorem{definition}{Definition}
\newtheorem{proposition}{Proposition}
\newtheorem{example}{Example}
\newtheorem{observation}{Observation}
\newtheorem{corollary}{Corollary}
\newtheorem{lemma}{Lemma}

\newcommand{\agents}{\mathcal{N}}

\newcommand*{\al}[1]{
	\begin{tikzpicture}
	[baseline=(letter.base)]
	\node[draw,circle,inner sep=1pt] (letter) {$\mathrm #1$};
	\end{tikzpicture}
}

\begin{document}

\title{Fair in the Eyes of Others}

\author{Parham Shams$^1$ \and Aur\'elie Beynier\institute{Sorbonne Universit{\'e}, France, email: name.surname@lip6.fr} \and Sylvain Bouveret\institute{Universit{\'e} Grenoble Alpes, France, email: sylvain.bouveret@imag.fr} \and Nicolas Maudet$^1$}

\maketitle
\bibliographystyle{ecai}

\begin{abstract}
  Envy-freeness is a widely studied notion in resource allocation, capturing some aspects of fairness. 
The notion of envy being inherently subjective though, it might be the case that an agent envies another agent, but that she objectively has no reason to do so. 
The difficulty here is to define the notion of objectivity, since no ground-truth can properly serve as a basis of this definition. A natural approach is to consider the judgement of the other agents as a proxy for objectivity.  
Building on previous work by Parijs (who introduced ``unanimous envy'') we propose the notion of approval envy: an agent $a_i$ experiences approval envy towards $a_j$ if she is envious of $a_j$, and sufficiently many agents agree that this should be the case, from their own perspectives. Some interesting properties of this notion are put forward. Computing the minimal threshold guaranteeing approval envy clearly inherits well-known intractable results from envy-freeness, but (i) we identify some tractable cases such as house allocation; and (ii) we provide a general method based on a mixed integer programming encoding of the problem, which proves to be efficient in practice. This allows us in particular to show experimentally that existence of such allocations, with a rather small threshold, is very often observed. 
\end{abstract}

\section{\uppercase{Introduction}}

Fair division is an ubiquituous problem in multiagent systems, economics \cite{Steinhaus48,Moulin03,Young94}, with applications ranging from allocation of schools, courses or rooms to students \cite{Abraham2005,Othman2010}, division of goods in inheritance or divorce settlement \cite{Brams96}. 
Envy-freeness (EF), is one of the prominent notions studied in fair division \cite{Foley67,BramsFishburn02,Lipton04,KBKZ-ADT09,Segal-HaleviS19}. An allocation of items among a set of agents is said to be envy-free if no agent prefers the share of another agent to her own share.
Unfortunately, envy-freeness is a pretty demanding notion and an envy-free allocation may not exist.  

Now consider a given problem where no envy-free allocation can be returned, but suppose instead that two allocations make a single agent (say, $a_i$) envious of some other agent $a_j$ (for simplicity). Now assume that in allocation $\pi$, $a_i$ is the only agent to prefer the bundle of $a_j$ over her own, while in allocation $\pi'$ all the other agents agree on the fact that $a_i$ should indeed envy $a_j$. According to Parijs \cite{parijs1997real}, $\pi'$ exhibits \emph{unanimous envy}, and there seems to be no situation where $\pi'$ should be returned in place of $\pi$.
Inspired by this notion, we introduce in this paper the notion of $K$-approval envy, as a way to retrieve a continuum between envy-freeness and unanimous envy. As may be clear from the name, the idea is simply to ask agents to express their own view about envy relations expressed by other agents. The objective will thus be to seek allocations minimizing social support for the expressed envy relations, i.e. minimizing the number of agents $K$  approving the envy. 
Of course, this approach may be controversial: after all, the notion of preference is inherently subjective. Introducing this flavour of objectivity may lead to undesirable consequences. 
At the extreme, one may simply replace individual preferences by some unanimous ``mean'' profile, thus profoundly changing the very nature of the notion. 
We believe there are several justifications to investigate this new approach: 
\begin{itemize}
	\item First, note that we only seek the approval of other agents in the case the agent herself explicitly expresses envy: absence of envy thus remains completely subjective. While a symmetrical treatment may also be justifiable in some situations, there is an obvious reason which motivates us to start with the proposed definition, namely the fact that the notion would no longer be a relaxation of envy-freeness. 
	\item Secondly, all other things being equal, we believe an allocation minimizing $K$ is socially more desirable.  
	We do not necessarily regard this notion as a compelling choice, but we think this can enrich the picture of fallback allocations when no envy-free allocation exists, as other relaxations do \cite{amanitidis-ijcai18}.  
	\item Finally, one further motivation of our work that we would like to emphasize is that our approach can be seen as providing guidance regarding agents and more specifically agents’ preferences which could be
	focused on, in order to progress towards envy-freeness. In particular, if we envision systems integrating
	deliberative phases in the collective decision-making process, our model could be used to set the agenda of such deliberations. If a vast majority of agents contradict an agent on her envy towards another agent, it may indicate for instance that she lacks information regarding the actual value of (some items of) her share. Initiating a discussion might help to solve such ``objectively unjustified'' envies when they occur. 
\end{itemize}


\paragraph{Outline of the paper.}
The remainder of this paper is as follows. Section \ref{sec:modele} recalls some basic notions in fair division. Our notion of $K$-approval envy is presented in Section \ref{sec:notion}. Some properties of this notion are then studied in Section \ref{sec:properties}: it is shown in particular, that if the hypothetical situation of allocation $\pi$ described at the beginning of the introduction occurs, then an EF allocation must also exist. We also show that our notion inherits from the complexity of related problems.   
This motivates the MIP formulation that we detail in Section \ref{sectionMIP}.  
We next turn to the House Allocation setting and we show that if each agent exactly holds a single item, then an efficient algorithm allows for returning an allocation minimizing the value of $K$. 
One caveat of our notion is that (unlike other relaxations) it is not guaranteed to exist, as intuitively observed in the case of unanimous envy. We thus consider greatly important to provide empirical evidence showing that both in different synthetic cultures as well as with real datasets, allocations with reasonable values of $K$ exist.

\section{\uppercase{Model and Definitions}}
\label{sec:modele}

We consider MultiAgent Resource Allocation problems (MARA) where we aim at  fairly dividing a set of  indivisible goods (also called items or objects) among a set of agents. A MARA instance  $I$ is defined as a finite set of
\emph{objects} $\mathcal{O} = \{o_1,\dots,o_m\}$,  a finite set of
\emph{agents} $\mathcal{N} = \{a_1,\dots,a_n\}$ and a profile $ \mathcal{P}$ of preferences  representing the interest of each agent $a_i \in \mathcal{N}$  towards the objects. An allocation ${\pi} $ is a mapping of the objects in $\mathcal{O}$ to the agents in $\mathcal{N}$. In the following, $\pi_i$ will denote the set of objects (the share) held by agent $a_i$. An allocation is such that $\forall a_i, \forall a_j$ with
$i \neq j : \pi_i\cap\pi_j=\emptyset$
(a given object cannot be allocated to more than one agent) and
$\bigcup_{a_i\in\mathcal{N}} \pi_i = \mathcal{O}$ (all
the objects from $\mathcal{O}$ are allocated).

In this paper, we consider cardinal preference profiles so, the preferences of an agent $a_i$ over bundles of objects is defined by a \emph{utility function} $u_i:2^\mathcal{O} \to \mathbb{Q}^+$ 
measuring her satisfaction $u_i(\pi_i)$ when she obtains share $\pi_i$. We make the assumption that utility functions are additive \textit{i.e.} the utility of an agent $a_i$ over a share $\pi_i$ is defined as the sum of the utilities over the objects forming  $\pi_i$:
\begin{equation*} \label{eq:utilite} u_i(\pi_i) \Def
\sum_{o_k\in\pi_i}u(i,k),
\end{equation*}
where $u(i, k)$ is the utility given by agent $a_i$ to object
$o_k$. This assumption is commonly considered in MARA  \cite[for instance]{Lipton04,Procaccia14,Dickerson2014,Caragiannis:2016} as additive utility functions provide a compact but yet expressive way  to represent the preferences of the agents. MARA instances with additive utility functions are called add-MARA instances for short.



Different notions  have been proposed in the literature to value the fairness of an  allocation. When the agents can compare their shares, the absence of envy  \cite{Foley67,Lipton04,chevaleyre_reaching_2007} is a particularly relevant notion of fairness. An agent $a_i$ would envy another agent $a_j$ if she prefers the share of $a_j$ over her own share. More formally, an agent $a_i$ envies an agent $a_j$ iff
\[
u_i(\pi_j) > u_i(\pi_i)
\]

A completely fair allocation would thus be an \emph{envy-free allocation} \textit{i.e.} an allocation where no agent envies another agent. Formally:
\[
\forall a_i, a_j \in \mathcal{N}, u_i(\pi_i) \geq u_i(\pi_j)
\]

The notion of envy-freeness conveys a natural concept of fairness viewed as social stability: agents are happy with their bundle and hence would not want to swap it with any other agent's (regarding their own preferences). However, as soon as it is required to allocate all the objects in $\mathcal{O}$, an envy-free allocation may not exist. An alternative objective may be to minimize a degree of envy of the society \cite{Lipton04,chevaleyre_reaching_2007}, based on the notion of \emph{paiwise envy}. 

\begin{definition}[Pairwise envy]
	Let $\pi$ be an allocation. The pairwise envy $pe(i,j,{\vect{\pi}})$ of an agent $a_i$ towards an agent $a_j$ in $\pi$ is defined as follows:
	\[
	pe(i,j,{\pi}) \Def  \max\{0,u_i(\pi_j)-u_i(\pi_i)\}.
	\]
\end{definition}

The pairwise envy can be interpreted as how much agent $a_i$ envies agent $a_j$'s share (this envy being 0 if $a_i$ does not envy $a_j$). We can derive from this notion a collective measure of envy:

\begin{definition}[Degree of envy of the society]
	The degree of envy of the society for an allocation $\pi$ is defined as follows:
	\[
	de({\pi}) \Def \sum_{a_i\in\mathcal{N}}\sum_{a_j\in\mathcal{N}}pe(i,j,{\pi})
	\]
\end{definition}

Note that an allocation $\pi$ is envy-free if and only if $ de({\pi}) = 0$.

To cope with the possible inexistence of an envy-free allocation, another approach is to alleviate the requirements of the fairness notion. 
Recently, several relaxations of envy-freeness have been proposed  such as envy-freeness up to one good  (EF1) \cite{Budish11},  envy-freeness up to any good (EFX) \cite{Caragiannis:2016}. 
An allocation is said to be envy-free up to one good (resp. up to any good) if no agent $a_i$ envies the share $\pi_j$ of another agent $a_j$ after removing from $\pi_j$  \emph{one} (resp. any) item. Existence for EF1 is guaranteed, but this is still to the best of our knowledge an open question for EFX.  
Amanitidis et al. \cite{amanitidis-ijcai18} studied the relations between some fairness notions and their relaxations. 

\section{\uppercase{$K$-approval envy}}
\label{sec:notion}



The notion of envy being inherently subjective, it might be the case that an agent envies another agent, but that she objectively has no reason to do so. The difficulty here is to define the notion of objectivity, since no ground-truth can properly serve as a basis of this definition. In her book, Guibet-Lafaye \cite{lafaye2006justice} recalls the notion of \emph{unanimous envy}, that was initially discussed in the book by Parijs \cite{parijs1997real}, and that can be defined as follows: an agent $a_i$ unanimously envies another agent $a_j$, if all the agents think that $a_i$ indeed envies $a_j$. Here, unanimity is used as a proxy for objectivity.

As we can easily imagine, looking for allocations that are free of unanimous envy will be too weak to be interesting: as soon as one agent disagrees with the fact that $a_i$ envies $a_j$, this potential envy will not be taken into account. Here, we propose an intermediate notion between envy-freeness and (unanimous envy)-freeness:

\begin{definition}[$K$-approval envy]\label{def:kappenvy}
	Let $\pi$ be an allocation, $a_i, a_j$ be two different agents, and  $1 \leq K \leq n$ be an integer. We say that \emph{$a_i$ $K$-approval envies ($K$-app envies for short) $a_j$} if there is a subset $\mathcal{N}_K$ of $K$ agents including $a_i$ such that:
	\[
	\forall a_k \in \mathcal{N}_K, u_k(\pi_i) < u_k(\pi_j).
	\]
	In other words, at least $K-1$ agents amongst $\mathcal{N} \setminus \{a_i\}$ agree with $a_i$ on the fact that she should actually envy agent $a_j$.
\end{definition}

\begin{example}
	\label{first-example}
	Let us consider the following add-MARA instance with 3 agents and 6 objects:
	\[
	\begin{array}{c|cccccc}
	& o_1 & o_2 & o_3 & o_4 & o_5& o_6 \\
	\hline
	a_1 & 0 & \fbox{3} & 3 & 1 &  3 & \fbox{2}   \\
	a_2 & \fbox{2} &0  &7 & \fbox{2} & \fbox{1} & 0    \\
	a_3 & 0 & 3 & \fbox{5} & 0 & 1 & 3\\
	\end{array}
	\]
	Note that there is no envy-free allocation for this instance. In the squared allocation, $a_1$ is not envious,  $a_2$ envies $a_3$ and $a_3$ envies $a_1$. Concerning the envy of $a_2$ towards $a_3$, $a_1$ disagrees with $a_2$ being envious of $a_3$ whereas agent $a_3$ agrees. Hence, agent $a_2$ 2-app envies agent $a_3$. Concerning the envy of $a_3$ towards $a_1$, agent $a_1$ agrees with $a_3$ being envious of $a_1$ whereas agent $a_2$ does not. Hence, $a_3$ 2-app envies $a_1$.
\end{example}

Note that in the definition, as soon as $a_i$ does not envy $a_j$, then, $a_i$ does not $K$-app envy $a_j$, no matter what the value of $K$ is or how many agents think that $a_i$ should actually envy $a_j$.

Let us start with an easy observation:

\begin{observation}
	\label{observation:KappImplK-1app}
	Given an allocation $\pi$ of an add-MARA instance, if $a_i$ $K$-app envies $a_j$ in $\pi$, then $a_i$ $(K\!\!-\!1)$-app envies $a_j$ in $\pi$.  
\end{observation}

Moreover, if $a_i$ $n$-app envies $a_j$, we will say that $a_i$ \emph{unanimously} envies $a_j$. Finally, we can observe that $a_i$ $1$-app envies $a_j$ if and only if $a_i$ envies $a_j$.

We can naturally derive from Definition~\ref{def:kappenvy} the counterpart of envy-freeness:

\begin{definition}[($K$-approval envy)-free allocation]
	An allocation $\pi$ is said to be ($K$-app envy)-free if and only if $a_i$ does not $K$-app envy $a_j$ for all pairs of agents $(a_i, a_j)$. 
\end{definition}

\begin{definition}[($K$-approval envy)-free instance]
	An add-MARA instance $I$ will be said to be ($K$-app envy)-free if and only if it accepts a ($K$-app envy)-free allocation. 
\end{definition}

\begin{example}
	Going back to Example~\ref{first-example}, the squared allocation is ($3$-app envy)-free so the instance is ($3$-app envy)-free. 
\end{example}

A threshold of special interest is obviously $\lfloor n/2\rfloor + 1 $, since it requires a strict majority to approve the envy under inspection.
A Strict Majority approval envy-free (SM-app-EF) allocation is a ($K$-app envy)-free allocation such that $K\leq \left\lceil n/2\right\rceil$, translating the fact that every time envy occurs, there is a strict majority of agents that do not agree with that envy.

Going further, it is important to notice that ($K$-app envy)-freeness is not guaranteed to exist. 
Indeed, for all number of agents $n$ and all number of objects $m$, there exist instances for which no ($K$-app envy)-free allocation exists, no matter what $K$ is.
Suppose for instance that all the agents rank the same object (say $o_1$) first, and that for all $a_i$, $u(i,1) > \sum_{k = 2}^m u(i,k)$. Then obviously, everyone agrees that all the agents envy the one that will receive $o_1$.
Such instances will be called \emph{unanimous envy instances}:

\begin{definition}[Unanimous envy instance]
	\label{udef:nanimous-envy}
	An add-MARA instance $I$ will be said to exhibit unanimous envy if $I$ is not
	($K$-app envy)-free for any value of $K$.
\end{definition}

Observe that for an allocation to be ($K$-app envy)-free, for all pairs of agents $a_i$, $a_j$, either $a_i$ or at least $n-K+1$ agents have to think that $a_i$ does not envy $a_j$. Notice that it is different from requiring that at least $K$ agents think that this allocation is envy-free. This explains the parenthesis around ($K$-app envy): this notion means ``free of $K$-app envy'', which is different from ``$K$-app-(envy-free)''.



A useful representation, for a given allocation, is the induced envy graph: vertices are agents, and there is a directed edge from $a_i$ to $a_j$ if and only if $a_i$ envies $a_j$ \cite{Lipton04}. An allocation is envy-free if and only if the envy graph has no arc. In our context, we can define a weighted notion of the envy graph. 

\begin{definition}[\textbf{Weighted envy graph}]
	\label{def:WeightedGraphEnvy}
	The \emph{weighted envy graph} of an allocation $\pi$ is defined as the weighted graph $(\agents,E)$ where nodes are agents, such that there is an edge $(a_i,a_j) \in E$ if $a_i$ envies $a_j$, with the weight $w(a_i,a_j)$ corresponding to the number of agents (including $a_i$) approving this pairwise envy in $\pi$.  
\end{definition}






Our notion of $K$-approval envy can be interpreted as a vote on envy, that works as follows. For each pair of agents $(a_i, a_j)$, if $a_i$ declares to envy $a_j$, we ask the rest of the agents to vote on whether they think that $a_i$ indeed envies $a_j$. Then, a voting procedure is used to determine whether $a_i$ envies $a_j$ according to the society of agents. Several voting procedures can be used. However, since there are only two candidates (yes / no), the most reasonable voting rules are based on quotas: $a_i$ envies $a_j$ if and only if there is a minimum quota of agents that think so.\footnote{More precisely, these rules exactly characterize the set of anonymous and monotonic voting rules \cite{perry2010anonymity}.}

\section{\uppercase{Some properties of $K$-app envy}}
\label{sec:properties}


There are natural relations between the different notions of ($K$-app envy)-freeness, for different values of $K$. The following observation is a direct consequence of Observation~\ref{observation:KappImplK-1app}.

\begin{observation}
	\label{prop:kimplk+1}
	Let $\pi$ be an allocation, and $K \leq N$ be an integer. If $\pi$ is ($K$-app envy)-free, then $\pi$ is also ($(K\!\!+\!1)$-app envy)-free.
\end{observation}


However, the converse does not hold. More precisely, the following proposition shows that the implication stated in Observation~\ref{observation:KappImplK-1app} may be strict.
\begin{proposition}
	\label{prop:KallocNotImplyK-1alloc}
	Let $\pi$ be an allocation, and $3 \leq K \leq n$ be an integer. If $\pi$ is ($K$-app envy)-free, $\pi$ is not necessarily ($(K\!\!-\!1)$-app envy)-free.
\end{proposition}

\hidden{
	\begin{proof}
		Let us consider the following instance with 4 agents and 4 objects and the squared allocation $\pi$:
		\[
		\begin{array}{c|ccccc}
		& o_1 & o_2 & o_3 & o_4 \\
		\hline
		a_1 & \fbox{$1$} & 0 & 0 & 0   \\
		a_2 & \frac{1}{2} & \fbox{$\frac{1}{2}$} & 0 & 0    \\
		a_3 & 0 & 0 & \fbox{$1$} & 0    \\
		a_5 & \frac{2}{5} & \frac{1}{5} & \frac{1}{5} & \frac{1}{5} & \fbox{$\frac{1}{5}$}   \\
		\end{array}
		\]
		
		In this allocation, the only envy concerns $a_4$ towards $a_1$. Moreover, only $a_1$ and $a_2$ agree with $a_4$ on her envy. Hence, $\pi$ is (4-app envy)-free but is obviously not (3-app envy)-free as we can find 3 agents ($a_1$, $a_2$ and $a_4$) agreeing on the envy of $a_4$ towards $a_1$ (in other words $a_4$ 3-app envies $a_1$).
		This example can be easily generalized for any number of agents. Indeed, if we note $h$ the number of agents that value strictly positively $o_1$. Then, we can build the allocation for which an allocation is ($(h\!+\!1)$-app envy)-free but not ($h$-app envy)-free. This concludes the proof.
	\end{proof}
}

\begin{proof}
	\label{proof:KallocNotImplyK-1alloc}
	Let $h \in \{2, \dots, n-1\}$  be an integer, and let us consider the instance with $n$ agents and $n$ objects defined as follows:
	\begin{itemize}
		\item $u(1, 1) = 1$;
		\item $u(i, 1) = u(i, i) = \frac{1}{2}$ for $i \in \{2, \dots, h-1\}$;
		\item $u(i, i) = 1$ for $i \in \{h, n-1\}$;
		\item $u(n, 1) = \frac{2}{n+1}$ and $u(n, j) = \frac{1}{n+1}$ for $j > 1$;
	\end{itemize}
	and $u(i, j) = \varepsilon$ for other pairs with $\varepsilon<\frac{1}{n+1}$.
	
	Consider the allocation $\pi$ where each agent $a_i$ gets item $o_i$. Obviously, the only envy in this allocation concerns $a_n$ towards $a_1$. Moreover, only $a_1, \dots, a_{h-1}$ agree on this envy. Therefore, $a_n$ $h$-app envies $a_1$, but does not $(h\!+\!1)$-app envy her. Moreover, $\pi$ is ($(h\!+\!1)$-app envy)-free, but not ($h$-app envy)-free.
\end{proof}

\begin{example}
	In order to illustrate the previous proof, let us consider the following instance with 4 agents, 4 objects (and $h$=3) and the squared allocation $\pi$:
	\[
	\begin{array}{c|cccc}
	& o_1 & o_2 & o_3 & o_4 \\
	\hline
	a_1 & \fbox{$1$} & \varepsilon & \varepsilon & \varepsilon   \\
	a_2 & \frac{1}{2} & \fbox{$\frac{1}{2}$} & \varepsilon & \varepsilon    \\
	a_3 & \varepsilon & \varepsilon & \fbox{$1$} & \varepsilon    \\
	a_4 & \frac{2}{5} & \frac{1}{5} & \frac{1}{5} & \fbox{$\frac{1}{5}$}   \\
	\end{array}
	\]
	In this allocation, the only envy concerns $a_4$ towards $a_1$. Moreover, only $a_1$ and $a_2$ agree with $a_4$ on her envy. Hence, $\pi$ is (4-app envy)-free but is obviously not (3-app envy)-free as we can find 3 agents ($a_1$, $a_2$ and $a_4$) agreeing on the envy of $a_4$ towards $a_1$ (in other words $a_4$ 3-app envies $a_1$).
\end{example}

\begin{proposition}
	\label{prop:stricthierarchy}
	For any $K \geq 3$, there exist instances which are ($K$-app envy)-free but not ($(K\!\!-\!1)$-app envy)-free.
\end{proposition}

\begin{proof}
	Consider the same instance as in Proposition \ref{prop:KallocNotImplyK-1alloc}. We have already shown that we have an allocation $\pi$ that is ($(h\!+\!3)$-app envy)-free which means that the instance is ($(h\!+\!3)$-app envy)-free. We just have to show that there is no ($(h\!+\!2)$-app envy)-free allocation in order to conclude.
	In that purpose, we first note that each agent has to get one and exactly one object. Indeed, if it is not the case at least one agent $a_i$ will have no object and will thus be envious of any agent $a_j$ that has an object. Moreover, as all agents value the empty bundle with utility 0 and every object is valued with a strictly positive utility, this envy will be unanimous. Hence, each agent has to get one and exactly one object in order to minimize the ($K$-app envy)-freeness.
	Now consider objects $o_j$ for $j \in \{h+2, n\}$. The agents $a_j$ that receive an object $o_j$ and that are envious will $(h\!\!+\!2)$-app envy the agent that received $o_1$. Indeed, agents $a_i$ for $i \in \{h+2, n-1\}$ value objects $o_j$ with a utility higher than (or equal to) the one of $o_1$ (and thus do not approve the envy) while it is the opposite for the other agents who are exactly $h\!\!+\!2$ hence the $(h\!\!+\!2)$-app envy. 
	So if we want to avoid that envy, we have to give the objects $o_j$ to agents so that they do not experience envy at all but it is not possible as such agents are agents $a_p$ for $p \in \{h+2, n-1\}$. It means that we have $n-1-(h+2)+1$ agents that have to receive one of the $n-(h+2)+1$ objects which is obviously impossible. This means that we cannot avoid $(h\!\!+\!2)$-app envy which implies that no allocation can be ($(h\!\!+\!3)$-app envy)-free. 
\end{proof}

Proposition~\ref{prop:stricthierarchy} proves that the hierarchy of $K$-app envy instances is strict for $K \geq 3$. Rather surprisingly, we will see that it is not the case for $K = 2$.

In order to show this result, we will resort to a tool that has been proved to be really useful and powerful in many contexts dealing with envy \cite{BiswasEtAl-2018,amanatidis2019multiple,beynierEtAl:2019}:
the ``bundle reallocation cycle technique''. This technique, originating from the seminal work of Lipton {\it et al.} \cite{Lipton04}, consists in performing a cyclic reallocation of \emph{bundles} so that every agent is strictly better in the new allocation. Thus, such a reallocation corresponds to a cycle in the opposite direction of the edges in the --- weighted --- envy graph introduced in Definition~\ref{def:WeightedGraphEnvy}. 
It is known that performing a reallocation cycle decreases the degree of envy \cite{Lipton04}. Unfortunately, our first remark is that it does not necessarily decrease the level of $K$-app envy. Worse than that, it can actually increase it:



\begin{proposition}
	\label{lipton-doesnt-work}
	Let $\pi$ be a ($K$-app envy)-free allocation, for $3\leq K\leq n-1$. After performing an improving bundle reallocation cycle (even between two agents), the resulting allocation may be ($K'$-app envy)-free (and not ($K$-app envy)-free) such that $K'>K$.
\end{proposition}

\begin{proof}
	Let $h \in \{0, \dots, n-3\}$  be an integer, and let us consider the instance with $n$ agents and $n$ objects defined by the following utility functions:
	\begin{itemize}
		\item $a_1$: $u(1, 1) = 1$,$u(1, 2) = 2$,$u(1, 3) = 7$;
		\item $a_2$: $u(2, 1) = 2$,$u(2, 2) = 2$;
		\item $a_3$: $u(3, 3) = 10$;
		\item $a_l$ for $l \in \{4, h\}$: $u(l, 1) = u(l, 3) = 5, u(l, j) = 6$ for $j\geq 4$;
		\item $a_m$ for $m \in \{h+4, n\}$: $u(m, 2) = 4,u(m, 3) = 5, u(m, i) = 6$ for $i\geq 4$;
	\end{itemize}
	and $u(i, j) = 0$ for other pairs.
	
	Consider the allocation $\pi$ where each agent $a_i$ gets item $o_i$. Obviously, the only envy in this allocation concerns $a_1$ towards $a_2$ (approved by $a_1$ and agents $a_m$) and $a_3$ (approved by $a_1$, $a_3$ and agents $a_m$), and the envy of $a_2$ towards $a_1$ (approved by $a_2$ and agents $a_l$). Hence the allocation is ($(\max\{|a_m|\!\!+\!\!3,|a_l|\!\!+\!\!2\})$-app envy)-free.
	We now consider the allocation $\pi'$ resulting from the improving bundle reallocation cycle between $a_1$ and $a_2$. We note that the only envy in $\pi'$ is the one of $a_1$ towards $a_3$. Moreover, this envy is approved by herself, $a_3$ and agents $a_l$ and $a_m$. The allocation is thus ($(|a_m|\!+\!|a_l|\!+3)$-app envy)-free and not ($\max\{|a_m|\!+\!3,|a_l|\!+\!2\}$-app envy)-free  as if $|a_l|\!\geq\!1$ then $|a_m|\!+\!|a_l|\!+3>\max\{|a_m|\!+\!3,|a_l|\!+\!2\}$.
\end{proof}

Now consider a slight generalization of  Lipton's cycles, \emph{weakly improving cycles}, that correspond to a reallocation of bundles where all the agents in the cycle receive a bundle they like at least as much as the one they held, with one agent at least being strictly happier. Of course, our example of Proposition \ref{lipton-doesnt-work} still applies. On the other hand, this notion suffices to guarantee the decrease of the degree of envy (note that \emph{identifying} the cycles themselves may not be easy any longer, but this is irrelevant for our purpose). The proof, omitted, follows directly from the arguments of Lipton \cite{Lipton04}.

\begin{observation}
	\label{obs:lipton-generalized}
	Let $\pi$ be an allocation, and $\pi'$ the allocation obtained after performing a \emph{weakly improving cycle}. It holds that $de(\pi') < de(\pi)$. 
\end{observation}


We now show that ($2$-app envy)-freeness exhibits a special behaviour: in contrast with Proposition \ref{lipton-doesnt-work}, improving cycles (in fact, even weakly improving cycles) enjoy the property of preserving the ($2$-app envy)-freeness level of an allocation. We provide this result for \emph{swaps} (cycles involving two agents only) as this is sufficient to establish our main result. 

\begin{lemma}
	\label{lemme1}
	Let $\pi$ be a ($2$-app envy)-free allocation. After performing a weakly improving bundle reallocation cycle between two agents, the resulting allocation is ($K'$-app envy)-free, with $K' \leq 2$. 
\end{lemma}

\begin{proof}
	Let us consider the two agents $a_i, a_j$ that are involved in the weakly improving bundle reallocation cycle. We respectively call $\pi$ and $\pi'$ the initial and the resulting allocations (hence $\pi_i=\pi'_j$,$\pi_j=\pi'_i$). First note that apart from $a_i$ and $a_j$, the approval envy of the agents does not change as their bundle remains the same from $\pi$ to $\pi'$. So, we just have to show that it is not possible for $a_i$ (or $a_j$) to experience $K$-app envy with $K\geq2$ in $\pi'$ (w.l.o.g. we can focus on $a_i$ only as the same proof holds for $a_j$). Indeed, if such $2$-app envy exists in $\pi'$, it means that this allocation would be at least ($3$-app envy)-free.
	Let us consider for the sake of contradiction that $a_i$ does indeed experience $K$-app envy with $K\geq2$. This means that there is an agent $a_h$ that $a_i$ $K$-app envies, i.e. $u_i(\pi'_i)<u_i(\pi'_h)$ and that there is some agent $a_l$ that approves this envy such that  $u_l(\pi'_i)<u_l(\pi'_h)$. Note that $\pi'_h=\pi_h$ as $a_h$ can obviously be neither $a_i$ neither $a_j$ so the bundle of $a_h$ remained the same. However, we know $u_l(\pi_i)\geq u_l(\pi_j)$ (otherwise $a_i$ would have $2$-app envied $a_j$ in $\pi$ thus contradicting the ($2$-app envy)-freeness of $\pi$) and $u_l(\pi_j)\geq u_l(\pi_i)$ (otherwise $a_j$ would have $2$-app envied $a_i$ in $\pi$ thus contradicting the ($2$-app envy)-freeness of $\pi$) which obviously implies $u_l(\pi_i) = u_l(\pi_j)$ but also equals to $u_l(\pi'_i)$ and $u_l(\pi'_j)$. So $u_l(\pi'_i)<u_l(\pi'_h) \Longleftrightarrow u_l(\pi_i)<u_l(\pi_h)$ which is impossible because otherwise $a_i$ would have $2$-app envied $a_h$ in $\pi$ thus contradicting the ($2$-app envy)-freeness of $\pi$.  
\end{proof}


Putting together Lemma \ref{lemme1} and Observation \ref{obs:lipton-generalized} allows us to prove that (2-app envy)-freeness is essentially a vacuous notion, in the sense that any instance enjoying an allocation with this property will have an EF allocation as well. 

\begin{proposition}
	\label{proposition:prop2appEF}
	If an add-MARA instance is ($2$-app envy)-free then it is also envy-free.
\end{proposition}

\begin{proof}
	Take $\pi$ as being an arbitrary ($2$-app envy)-free allocation. First note that if there is no envious agent in $\pi$ then, by definition, $\pi$ is envy-free and the proposition holds. Suppose that $a_p$ envies $a_q$. Observe that $a_q$ cannot agree with $a_p$, because otherwise $\pi$ would not be (2-app envy)-free. Hence, we can perform a weakly improving bundle reallocation cycle between those two agents and call the resulting allocation $\pi'$. If $\pi'$ is envy-free then we are done. Otherwise, thanks to Lemma~\ref{lemme1}, we know that $\pi'$ is still (2-app envy)-free, and by Observation \ref{obs:lipton-generalized} that the degree of envy has strictly decreased. We can apply the same argument as above with two agents envying each other and swap their bundles. The process stops when the current allocation is envy-free. The process is guaranteed to stop because the degree of envy of the society is bounded below by zero and the degree of envy of the society decreases at each step until it equals zero (which corresponds to an envy-free allocation). 
\end{proof}
\hidden{
	\ps{A enlever du coup? : \\
		The previous result is clearly specific to ($2$-app envy)-freeness. In fact, it suffices to consider Example \ref{basic-example-ctd} \sylvain{This example comes at the very end of the paper (in Conclusion). Shouldn't it appear before?} to notice that the discussed allocation is indeed a (3-app envy)-free allocation, but that there is no EF allocation for this instance. 
		\nm{now that we have the new proposition this seems irrelevant, right?}
}}
Another consequence is that, for two agents, instances fall either in the envy-free or unanimous envy category: 
\hidden{
	\begin{example}
		We show here the example of a (3-app envy)-free allocation of an instance for which there is no envy-free allocation:\\
		\begin{center}
			$
			\begin{array}{c|ccc}
			& o_1 & o_2 & o_3\\
			\hline
			a_1 & 0 & 1 & \fbox{$0$}   \\
			a_2 & 0 & \fbox{$1$} & 0   \\
			a_3 & \fbox{$1$} & 0 & 0   \\
			\end{array}
			$
		\end{center}
		It is easy to check that the squared allocation is ($3$-app envy)-free $a_2$ and $a_3$ are both envy-free whereas $a_1$ is envious of $a_2$. This envy is approved by $a_2$ (and herself of course) but not $a_3$ who is indifferent between $o_2$ and $o_3$. This sets the ($3$-app envy)-freeness of the squared allocation.
		For an allocation to be envy-free, it would need $a_3$ to have $o_1$. It would let $o_2$ and $o_3$ for agents $a_1$ and $a_2$. The one getting $o_3$ would obviously be envious. Hence, there is no envy-free allocation.
	\end{example}
}

\begin{corollary}
	\label{corollary:2agentsEForUE}
	In the special case of 2 agents, if there is no envy-free allocation in $I$ then $I$ is a unanimous envy instance. 
\end{corollary}



\paragraph{Complexity.}
We conclude with a few considerations on the computational complexity of the problems mentioned so far. 
First of all, as  envy-freeness is ($1$-app envy)-freeness, the problem of finding the minimum $K$ for which there exists a ($K$-app envy)-free allocation is at least as hard as determining whether an envy-free allocation exists. 

One may also wonder how hard the problem of determining whether a given instance exhibits unanimous envy or not, {\it i.e.} whether a ($K$-app envy)-free allocation exists for \emph{some} value of $K$. 
For this question, instances where agents all have the same preferences provide insights. 

\begin{proposition}
	\label{prop:1appnapp}
	For any add-MARA instance, if all the agents have the same preferences then the notions of ($1$-app envy)-freeness and ($n$-app envy)-freeness coincide.
\end{proposition}

\begin{proof}
	We already know from Observation \ref{prop:kimplk+1} that ($1$-app envy)-freeness implies ($n$-app envy)-freeness for any add-MARA instance. So we just have to prove that if all the agents have the same preferences then ($n$-app envy)-freeness implies ($1$-app envy)-freeness. 
	If an allocation $\pi$ is ($n$-app envy)-free then it means that for any pair $a_i,a_j$ of agents,  $a_i$ does not envy $a_j$ or there is at least one agent $a_h$ that disagrees on the envy of $a_i$ towards $a_j$. 
	Obviously, if for every pair of agents $a_i$, $a_j$ we have $a_i$ envy-free towards $a_j$ then the allocation $\pi$ is envy-free and the proof concludes. 
	Besides, for every pair of envious/envied agents there is at least one agent disagreeing on the envy. But all the agents have the same preferences so it means that every agent should agree with each other. Hence, no envied agent can exist and we have ($1$-app envy)-freeness of allocation $\pi$.
\end{proof}

From Proposition \ref{prop:1appnapp} we get that the problem of deciding the existence of unanimous envy is at least as hard as deciding the existence of an EF allocation when agents have similar preferences. As membership in NP is direct, we thus get as a corollary that:
\begin{corollary}
	Deciding whether an allocation exhibits unanimous envy is NP-Complete. 
\end{corollary}
\hidden{
	However, while it is known that deciding EF is also hard for at least three agents in the even more restrictive domain of binary additive preferences \cite{EFNPCompleteBinary}, this does not hold when agents have similar preferences (the problem being trivial in that case, as an EF allocation exists iff the number of items liked by agents is divisible by $n$). Hence we cannot use the same argument to conclude on the hardness of our problem in this case. 
}

\section{\uppercase{A MIP formulation for $K$-app envy}}
\label{sectionMIP}

We have seen in the previous section that the problem of determining, for a given instance $I$, the minimal value of $K$ such that a ($K$-app envy)-free allocation exists inherited from the high complexity of determining whether an envy-free allocation exists.

To address this problem, we present in this section a Mixed Integer linear Program that returns, for a given add-MARA instance $I$, a ($K$-app envy)-free allocation with the minimal $K$ and no solution when $I$ is a unanimous envy instance.
In this MIP, we use $n \times m$ Boolean variables $\var{z_i^j}$ (we use bold letters to denote variables) to encode an allocation: $\var{z_i^j}=1$ if and only if $a_i$ gets item $o_j$. We also introduce $n^3$ Boolean variables $\var{e_{kih}}$ such that $\var{e_{kih}}=1$ if and only if according to $a_k$'s preferences $a_i$ envies agent $a_j$. We also need to add $n^2$ Boolean variables $\var{x_{ih}}$ used to linearize the constraints on $\var{e_{kih}}$. Finally, we use an integer variable $\var{K}$ corresponding to the $K$-app envy we seek to minimize.

In this section, we  assume that all the utilities are integers. If they are not (recall that they are still in $\mathbb{Q}^+$) we can transform the instance at stake into a new one only involving integral utilities by multiplying them by the least common multiple of their denominators. 
\hidden{
	\ab{Pas Sûre de bien comprendre l'explication, cela suppose qu'on ait des rationnels au départ ? on l'a dit quelque part ?}\sylvain{Oui, c'est moi qui ai modifié réels en rationnels dans la section 2. Sinon, ça pouvait poser des problèmes quand on parle de complexité.}
}

We first need to write the constraints preventing an item from being allocated to several agents: 

\begin{equation}
\label{constraint:allocation}
\sum_{i=1}^n \var{z_i^j} = 1 \qquad  \forall j \in \llb 1,m \rrb
\end{equation}

By adding these constraints we also guarantee completeness of the returned allocation (all the items have to be allocated to an agent).

Secondly, we have to write the constraints that link the variables $\var{e_{kih}}$ with the allocation variables $\var{z_i^j}$:

\begin{equation*}
\forall k,i,h \in \llb 1,n \rrb,
\sum_{j=1}^m u(k,j)(\var{z_h^j}-\var{z_i^j}) > 0 \implies \var{e_{kih}} = 1
\end{equation*}

As the utilities are integers, we can replace $>0$ by $\geq 1$. In order to linearize this implication between two constraints we introduce a number $M$ that can be arbitrarily chosen such that $M > \max_k \sum_{j = 1}^m u(k, j)$:

\begin{eqnarray}
\label{constraint:implicationLinearisation1}
M \var{e_{kih}} \geq \sum_{j=1}^m u(k,j)(\var{z_h^j}-\var{z_i^j}) & \forall k,i,h \in \llb 1,n\rrb\\
\label{constraint:implicationLinearisation2}
\sum_{j=1}^m u(k,j)(\var{z_h^j}-\var{z_i^j}) \geq 1 - M(1 - \var{e_{kih}}) & \forall k,i,h \in \llb 1,n\rrb
\end{eqnarray}

Finally, we have to write the constraints that convey the fact that the allocation we look for is ($K$-app envy)-free:

\begin{equation*}
\var{e_{iih}}=0 \lor \sum_{k=1}^n \var{e_{kih}} \leq \var{K} - 1 \qquad \forall i,h \in \llb 1,n\rrb
\end{equation*}

Since $\var{e_{iih}}$ are Boolean variables, we can replace $\var{e_{iih}} = 0$ by $\var{e_{iih}} \leq 0$. Now, this logical constraint is linearized as follows:

\begin{eqnarray}
\var{e_{iih}}\leq \var{x_{ih}} & \forall i,h \in \llb 1,n\rrb & \label{constraint:or1}\\
\sum_{k=1}^n \var{e_{kih}} \leq \var{K} - 1 + n(1 - \var{x_{ih}}) & \forall i,h \in \llb 1,n\rrb &\label{constraint:or2}
\end{eqnarray}

We can now put things together. Let $I$ be an instance. Then, we will denote by $\mathcal{M}(I)$ the MIP defined as:
\[
\begin{array}{ll}
\text{minimize} & \var{K}\\
\text{such that}
& \var{z_i^j}, \var{e_{kih}}, \var{x_{ih}} \in\{0, 1\} \forall k, i, h \in \llb 1, n \rrb, j \in \llb 1, m \rrb\\
& \var{K}\in\llb 1, N \rrb \\
& + \text{ Constraints (\ref{constraint:allocation}, \ref{constraint:implicationLinearisation1}, \ref{constraint:implicationLinearisation2}, \ref{constraint:or1}, \ref{constraint:or2})}
\end{array}
\]

\begin{proposition}
	Let $I$ be an instance. Then, there is an optimal solution with $\var{K} = L$ to $\mathcal{M}(I)$ if and only if $I$ is an ($L$-app envy)-free instance and not an
	(($L\!\!-\!1)$ envy)-free one. Moreover, $\mathcal{M}(I)$ does not admit any solution if and only if $I$ is an unanimous envy instance.
\end{proposition}

The proof of this proposition is not very involved and will thus be omitted. The key here is to show that there is a solution to the MIP $\mathcal{M}(I)$ such that $\var{K} = L$ iff the corresponding allocation $\pi$ such that $\var{z_i^j} = 1$ if and only if $o_j \in \pi_i$ is ($L$-app envy)-free. The most critical point is to show that Constraints~\ref{constraint:implicationLinearisation1} and \ref{constraint:implicationLinearisation2} are indeed a valid translation of the logical implication, and that Constraints~\ref{constraint:or1} and \ref{constraint:or2} correctly encode the logical or. The rest follows easily.

\hidden{
	\begin{proof}
		To prove the proposition,  we show that there is an ($L$-app envy)-free allocation in $I$ if and only if there is a solution to the MIP $\mathcal{M}(I)$ such that $\var{K} = L$.
		
		$(\Rightarrow)$ Let $I$ be an instance, and $\pi$ be an ($L$-app envy)-free allocation. Then, consider the partial instantiation of the variables such that $\var{z_i^j} = 1$ if and only if $o_j \in \pi_i$. We prove that this partial instantiation extends to a solution of $\mathcal{M}(I)$ such that $\var{K} = L$.
		
		First observe that Constraint~\ref{constraint:allocation} is directly satisfied.
		
		Now, consider any triplet of agents $(a_k, a_i, a_h)$. Suppose that agent $a_k$ thinks $a_i$ should envy $a_h$. Then in this case, we have $\sum_{j \in \pi_h} u(k, j) > \sum_{j \in \pi_i} u(k, j)$. In other words, $\sum_{j=1}^mu(k,j)(\var{z_h^j}-\var{z_i^j}) > 0$ which is in turn equivalent to $\sum_{j=1}^mu(k,j)(\var{z_h^j}-\var{z_i^j}) \geq 1$ since all utilities are integers.
		By Constraint~\ref{constraint:implicationLinearisation1}, we thus have that $\var{e_{kih}} = 1$ which implies that Constraint~\ref{constraint:implicationLinearisation2} is satisfied as well. 
		
		Conversely, suppose that agent $a_k$ thinks $a_i$ should not envy $a_h$. Then, we have $\sum_{j \in \pi_h} u(k, j) \leq \sum_{j \in \pi_i} u(k, j)$. In other words, $\sum_{j=1}^mu(k,j)(\var{z_h^j}-\var{z_i^j}) \leq 0$.
		By Constraint~\ref{constraint:implicationLinearisation2}, we thus have that $\var{e_{kih}} = 0$ in this case, which in turns implies that Constraint~\ref{constraint:implicationLinearisation1} is satisfied as well.
		Hence, we have that $\var{e_{kih}} = 1$ if and only if $a_k$ thinks $a_i$ should envy $a_h$ in $\pi$.
		
		Finally, consider any pair of agents $(a_i, a_h)$. If $a_i$ does not envy $a_h$ then $\var{e_{iih}} = 0$. As a consequence, we can have $\var{x_{ih}} = 0$ and still satisfy Constraints~\ref{constraint:or1} and \ref{constraint:or2} (no matter the  value of  $\var{K}$ is).
		
		Now suppose that $a_i$ does envy $a_h$ (hence $\var{e_{iih}} = 1)$. Then, we should have $\var{x_{ih}} = 1$ to satisfy Constraint~\ref{constraint:or1}. Since $\pi$ is ($L$-app envy)-free, then at most $L-1$ agents (including $a_i$ herself) think that $a_i$ should indeed envy $a_h$, which means that $\sum_{k=1}^n \var{e_{kih}} \leq L - 1$. Instantiating $\var{K}$ to $L$ is hence enough to satisfy Constraint~\ref{constraint:or2}.
		
		$(\Leftarrow)$ Now suppose that there is a solution to $\mathcal{M}(I)$ such that $\var{K} = L$. Then we will prove that the allocation $\pi$ such that $o_j \in \pi_i$ if and only if $\var{z_{i^j}} = 1$ is a valid ($L$-app envy)-free allocation.
		
		First, according to Constraints~\ref{constraint:allocation}, $\pi$ is indeed a valid allocation.
		
		Secondly, Constraint~\ref{constraint:implicationLinearisation1} ensures that if $\var{e_{kih}} = 0$ then $\sum_{j=1}^mu(k,j)(\var{z_h^j}-\var{z_i^j}) \leq 0$, in turn meaning that agent $a_k$ thinks that $a_i$ should not envy $a_h$. Conversely, Constraint~\ref{constraint:implicationLinearisation2} ensures that if $\var{e_{kih}} = 1$ then $\sum_{j=1}^mu(k,j)(\var{z_h^j}-\var{z_i^j}) > 0$, in turn meaning that agent $a_k$ thinks that $a_i$ should envy $a_h$. It also obviously implies that $\var{e_{iih}} = 1$ if and only if $a_i$ envies $a_h$.
		
		Now consider any pair of agents $(a_i, a_h)$ such that $a_i$ envies $a_h$. From what precedes, $\var{e_{iih}} = 1$. By Constraint~\ref{constraint:or1}, $\var{x_{ih}} = 1$. Hence, by Constraint~\ref{constraint:or2}, $\sum_{k = 1}^h \var{e_{kih}} \leq L - 1$. This implies that the total number of agents agreeing with the fact that $a_i$ envies $a_h$ is strictly lower than $L$. In other words, $\pi$ is  ($L$-app envy)-free.
	\end{proof}
}
\hidden{  
	We obtain this MIP:
	
	$$\min K$$
	$$
	\left\{
	\begin{array}{lllll}
	\displaystyle\sum_{i=1}^nz_i^j\!\!=\!\!1 &\!\!\!\forall j \in \llb 1,m\rrb\\
	\displaystyle\sum_{j=1}^mu(k,j)(z_h^j\!\!-\!z_i^j)\!\geq\!1\!\!-\!\!M(1\!\!-\!e_{kih})\!&\!\!\!\forall k,i,h \in \llb 1,n\rrb\\
	\displaystyle\sum_{j=1}^mu(k,j)(z_h^j-z_i^j) \leq Me_{kih} &\!\!\!\forall k,i,h \in \llb 1,n\rrb\\
	e_{iih}\leq x_{ih}&\!\!\!\forall i,h \in \llb 1,n\rrb\\
	\displaystyle\sum_{k=1}^n e_{kih}\leq K - 1 + n(1-x_{ih})&\!\!\!\forall i,h \in \llb 1,n\rrb\\
	\end{array}
	\right.
	$$
}


As the problem is difficult in the general case, it is natural to seek special cases that could be solved efficiently.

\section{\uppercase{House Allocation}}
\label{sectionPoly}

The House Allocation Problem (HAP for short) is a standard problem where there are exactly as many items as agents, and each agent receives exactly one resource. This setting is relevant in many situations and has been extensively studied \cite[to cite few of them]{shapley1974,Roth1992,Abraham2005}. 
In House Allocation Problems, computing an envy-free allocation comes down to solving a matching problem, since an envy-free allocation exists if and only if all the agents get (one of) their top item(s). It is therefore natural to wonder whether an allocation minimizing $K$-app envy could also be computed efficiently. 

Our first observation hints in that direction. Indeed, characterizing unanimous envy becomes easy in house allocation problems. 
\begin{proposition}
	Let $I$ be an instance of HAP. $I$ is an unanimous envy instance if and only if there exists at least a pair of items $(o_i, o_j)$ such that all agents strictly prefer $o_i$ over $o_j$.
	
\end{proposition}
\hidden{
	\begin{proof}
		($\Rightarrow$) Suppose no such pair $(o_i, o_j)$ of items exists. Let $\pi$ be any allocation giving to each agent $a_i$ an item (say $o_i$ w.l.o.g). For each pair of agents $(a_i, a_j)$ such that $a_i$ envies $a_j$. At least one agent thinks that $o_j$ is better than $o_i$, avoiding unanimous envy.
		
		($\Leftarrow$) In any allocation one agent $a_i$ holds $o_i$ while another agent $a_j$ holds $o_j$: $a_j$ envies $a_i$ and all the agents approve this envy.
	\end{proof}
	Incidentally, we get as a corollary: 
}

\begin{corollary}
	Checking whether an instance $I$ of HAP is a unanimous envy instance or not can be done in $O(n^2)$.
\end{corollary}


From this characterization we can also derive a result on the likelihood that unanimous envy exists when the utilities are uniformly ditributed
(that is, for each agent $a_i$ and object $o_j$, utilities are drawn i.i.d. following the uniform distribution on some interval $[x,y]$).

\begin{proposition}
	\label{prop:probaUnanimous}
	Under uniformly distributed preferences, the probability of unanimous envy is upper bounded by $n(n-1)/2^n$. 
\end{proposition}

\begin{proof}
	Wlog. suppose agent $1$ has preferences $o_1 \succ o_2 \succ \cdots \succ o_n$. The probability of the event $o_i$ is strictly preferred to $o_j$ by one agent is $1/2$ if preferences are strict. As preferences are not strict, this probability becomes an upper bound (think for instance if the agent values all the objects the same then the probability to have strict preference between two objects is zero). Hence, the probability of the event $o_i$ is strictly preferred to $o_j$ by all agents is upper bounded by $1/2^{n-1}$ as the preferences between the agents are independent. Assuming, for all pairs of items, these events to be independent (which is not the case, hence an upper bound of the upper bound), we derive our result by summing up over the $n(n-1)/2$ possible pairs.  
\end{proof}
Note that this value quickly tends towards $0$: unanimous envy is thus already unlikely to occur for 10 agents. 

\paragraph{}

We will now show here that finding an allocation that minimizes ($K$-app envy)-freeness can be done in polynomial time. Before introducing the idea, we need an additional notation. For any pair $(j, j')$, let $\#_\prec(j, j')$ denote the number of agents strictly preferring $o_{j'}$ to $o_{j}$. For any agent $a_i$ and object $o_j$, we will also define $maxEnvy[i][j]$ as follows:
\[
maxEnvy[i][j] = \max_{o_{j'} \text{ s.t. } u(i, j') > u(i, j)} \#_\prec(j, j')
\]
In other words, $maxEnvy[i][j]$ denotes the maximal value of $\#_\prec(j, j')$ among the objects that are strictly preferred to $o_j$ by $a_i$. As we can imagine, this will exactly be the value of the $K$-app envy experienced by $a_i$ if she gets item $o_j$ (note that if $o_j$ is among $a_i$'s top objects, this value will be 0).

The key to the algorithm is to see that for a given $K$, determining whether a ($K$-app envy)-free allocation exists can be done in polynomial time by solving a matching problem. Namely, for each $K$, we build the following bipartite graph: $(\mathcal{N}, \mathcal{O})$ is the set of nodes, and we add an edge $(a_i, o_j) \in \mathcal{N} \times \mathcal{O}$ if and only if $maxEnvy[i][j]$ is lower than or equal to $K$. We can observe that any perfect matching in this graph corresponds to a ($(K+1)$-app envy)-free allocation. 
The only thing that remains to do is to run through all possible values of $K$, which can be done by dichotomous search between $0$ and $n$. This is formalized in Algorithm~\ref{algo:algoPoly}.

\begin{algorithm}
	\SetKwInOut{Input}{input} \SetKwInOut{Output}{output} \SetKwData{None}{None}
	\SetKwFunction{computeMaxEnvy}{computeMaxEnvy}
	\SetKwFunction{buildBipartiteGraph}{buildBipartiteGraph}
	\SetKwFunction{perfectMatching}{perfectMatching}
	\caption{Minimizing ($K$-app envy)-freeness in the HAP}
	\label{algo:algoPoly}
	\Input{$I = \langle \mathcal{N}, \mathcal{O}, w \rangle$ a HAP instance}
	\hidden{
		\Output{Allocation $\pi$ and its level minimizing the ($K$-app envy)-freeness or \None if $I$ is a unanimous envy instance}
	}
	$maxEnvy \gets \computeMaxEnvy{}$\;
	
	res $\gets$ \None\;
	
	low, high $\gets$ $0,n$\;
	
	\While{low$\leq$ high}{
		$i \gets \left\lfloor (low+high)/2 \right\rfloor$\;
		$G \gets \buildBipartiteGraph{maxEnvy,i}$\;
		$\pi \gets \perfectMatching{G}$\;
		
		\If{$\pi$ is not \None}{
			res $\gets$ $\pi,i+1$\;
			
			high $\gets$ $i-1$\;
		}\Else{
			low $\gets$ $i+1$\;
		}
	}
	
	\Return{res}
\end{algorithm}

\begin{proposition}
	\label{prop:complHapPolyn3logn}
	For any HAP instance, we can find (one of) its optimal ($K$-app envy)-free allocations in $O(n^3\log(n))$.
\end{proposition}

\begin{proof}
	First, the computation of the matrix \emph{maxEnvy} runs in $O(n^3)$. Indeed, to compute $\emph{maxEnvy}[i][j]$ we first need to compute $\#_\prec(j, j')$ which already runs in $O(n^3)$ as we have to ask for each couple of objects ($n^2$ in total) the point of view of all the agents ($n$ in total). From that, as $maxEnvy[i][j] = \max_{o_{j'} \text{ s.t. } u(i, j') > u(i, j)} \#_\prec(j, j')$, we can compute $maxEnvy[i][j]$ in $O(n)$. As there are $n^2$ different pairs $(a_i,o_j)$ we have the final $O(n^3)$ complexity of computing \emph{maxEnvy}.
	
	Due to the dichotomous search, the algorithm needs to solve $\log(n)$ perfect matching problems, that can be solved in $O(n^3)$\cite{minouxgondran1984graphs}. The global complexity of Algorithm \ref{algo:algoPoly} is thus $O(n^3\log(n))$.
\end{proof}

\hidden{
	Of course, many allocations may enjoy the same level of ($K$-app envy)-freeness. However, from Lemma \ref{lemme1}, we know that it is possible to perform reallocation cycles on the returned allocation without degrading $K$ \ab{Cela ne me semble plus aussi évident maintenant}: in this house allocation setting, performing such cycles until no further improvement is possible will guarantee to return a Pareto-optimal allocation as bundles are composed of one and exactly one object. Note that this does not affect the overall worst-case complexity of the procedure.  
}

\section{\uppercase{Experimental results}}

We present here the results of the numerical tests we have conducted. These experiments serve two purposes: (i) evaluate the behaviour of the MIP we presented in Section \ref{sectionMIP} and of the polynomial algorithm described in Section \ref{sectionPoly}, and (ii) observe how our notion of $K$-app envy depends on the number of agents, of items, and on the type of preferences. All the tests presented in this section have been run on an Intel(R) Core(TM) i7-2600K CPU with 16GB of RAM and using the Gurobi solver to solve the Mixed Integer Program. We have tested our methods on three types of instances: Spliddit instances \cite{Spliddit}, instances under uniformly distributed preferences and instances under an adaptation  of Mallows distributions to cardinal utilities \cite{MallowsVMFCardinal}.

\subsection{Spliddit instances}
We have first experimented our MIP on real-world data from the fair division website Spliddit \cite{Spliddit}. There is a total of 3535 instances from 2 agents to 15 agents and up to 93 items. Note that 1849 of these instances involve 3 agents and 6 objects.
By running the MIP with a timeout of 10 minutes (after this duration the best current solution, if it exists, is returned) we were able to solve all the instances but 6. Among these 6 instances, 3 of them were HAP instances that we managed to  solve optimally with Algorithm~\ref{algo:algoPoly}. This only leaves us with 3 instances, for which the solver did return a solution but did not prove that it is optimal (within a timeout of 10 minutes).
Besides, 65\% of the instances are EF while 23\% of the instances exhibit unanimous envy. Moreover, 28\% of the remaining instances with more than 5 agents are SM-app EF.

\subsection{Uniformly distributed preferences}
\paragraph{General setting}
We also ran tests on instances under uniformly distributed preferences, with  $n$ varying from 3 to 10 and $m$ such that we produce settings where few EF allocations exist \cite{Dickerson2014}. For each problem size, we 
kept 60 instances that admit no EF allocation as we wanted to measure the behaviour of our notion when no such allocation exists (we know that if an EF allocation exists it will be returned by  our methods). As we are in the general setting we solved the instances via the MIP with a timeout of 60 seconds. 

The first three rows of Table~\ref{table1} respectively represent the percentage of instances that have been solved to optimal (a solution has been returned before the timeout), the percentage of unanimous envy instances and the percentage of SM-app-EF instances. We then have the mean value of $K/n$. Finally, we store the mean computation time (in seconds) of the instances (solved to optimal).

\begin{table}
	\begin{center}
		{\caption{Results of the experiments as a function of the number of agents.}\label{table1}}
		\scalebox{0.80}{
			\begin{tabular}{p{0.07\textwidth}|*{9}{l}}
				\toprule
				n & 2 & 3 & 4 & 5 & 6 & 7 & 8 & 9 & 10 \\
				\midrule
				\% OPT & 100 & 100 & 100 & 100 & 100 & 68.3 & 1.7 & 1.7 & 0 \\
				\% UEI & 100 & 21.7 & 5 & 0 & 0 & 0 & 0 & 0 & 0 \\
				\% SMAEF & 0 & 0 & 0 & 50 & 50 & 75 & 40 & 33.3 & 6.7 \\
				mean($K/n$) & NaN & 1 & 0.85 & 0.72 & 0.61 & 0.57 & 0.59 & 0.63 & 0.66 \\
				time(s) & $\varepsilon$ & 0.008 & 0.04 & 0.21 & 1.97 & 21.29 & 50.09 & 56.16 & NaN \\
				\bottomrule
			\end{tabular}
		}
	\end{center}
\end{table}

First note that considering 2 agents is a special case as shown in Corollary~\ref{corollary:2agentsEForUE}. Indeed, as we have removed the EF instances, all the remaining instances are unanimous envy ones. Moreover, we observe that the percentage of SM-app-EF allocations is zero for 4 agents. Indeed, an allocation is SM-app-EF for 4 agents  if there exists a ($K$-app envy)-free allocation such that $K\leq 2$. As we have removed all the EF instances,  we know (from Proposition~\ref{proposition:prop2appEF}) that we cannot find an SM-app-EF allocation.
The same holds for 3 agents. 

We can notice that the mean $K/n$ seems to be stabilising around 0.6. 
Besides, without any surprise, the computation time rapidly increases while the percentage of instances solved to optimal (under a timeout of 60 seconds) starts decreasing for 7 agents. Finally, positive results can be pinpointed: the very low percentage of unanimous envy instances, and the pretty high percentage of SM-app-EF ones.

\paragraph{House allocation}

We have also tested our polynomial algorithm on HAP instances under uniformly distributed preferences. We have generated 20 instances for each number of agents from 5 to 100 agents (and objects) by steps of 5.

\hidden{
	\begin{figure}
		\centering
		\includegraphics[scale=0.4]{graph/k_app_envy_HAP_Poly.png}
		\caption{Boxplot of the optimal value of ($k$-app envy)-free allocations in the HAP as a function of the number of agents and objects}
		\label{fig:k_app_envy_HAP_Poly}
	\end{figure}
}

\begin{figure}
	\centering
	\includegraphics[scale=0.6]{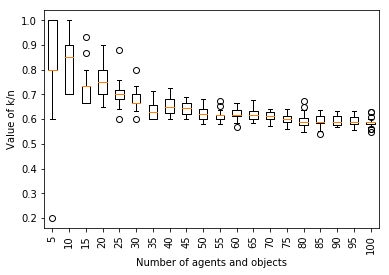}
	\caption{Optimal $K/n$ in the HAP as a function of $m$/$n$}
	\label{fig:k_app_dividedbyN_envy_HAP_Poly}
\end{figure}

First note that we have only found 5 unanimous envy instances and all of them  involved 5 agents. This supports the probability of unanimous envy instance showed in Proposition \ref{prop:probaUnanimous} and the fact that it decreases very quickly towards 0. Moreover, like for the general setting, we notice a convergence of the $K/n$ values towards 0.6.
The algorithm runs, without any surprise (in light of Proposition \ref{prop:complHapPolyn3logn}) much faster than our MIP. Indeed, the mean runtime for 100 objects and agents is still around 2 seconds only whereas we already observed that our MIP cannot solve easier problems within 10 minutes.

\subsection{Correlated preferences}

In strict ordinal settings, a classical way to capture correlated preferences is to use Mallows distributions \cite{MallowsModel57} allowing us to measure the impact of the similarity of the preferences between agents. 
In these experiments, we used a generalization of the Mallows distribution to cardinal preferences presented in \cite{MallowsVMFCardinal} based on Von Mises–Fisher distributions. 
Similarly to the dispersion parameter in Mallows distributions, the similarity between the preferences of the agents is tuned by the \emph{concentration} parameter: when it is zero agents' preferences are uniformly distributed, whereas when it is infinite agents have the same preferences. 

We expected that the more similar the preferences between the agents are, the higher the degree of $K$-app envy would be and the more likely unanimous envy would occur. The results of our experiments both  in the general setting and in HAP support this: the number of EF instances is decreasing along with the concentration value, and from a given threshold, all the instances exhibit unanimous envy. However, the exact correlation between the level of ($K$-app) envy-freeness and the concentration deserves further study, especially for very low values of $K$. Intuitively, in some circumstances, correlation of preferences may indeed help to find large majorities of agents that contradict an agent envy, while this situation is unlikely under uniformly distributed preferences. 
\hidden{
	\paragraph{General Setting}
	As the aim of the experiment is to show the impact of the concentration on the existence of our notion, we have fixed $n=3$ and $m=6$ and for each concentration value, we have generated 50 instances. The results of the experiments support our conjecture. Indeed, as soon as the concentration value exceeds 50 there are only unanimous envy instances while for lower values there was existence and even EF instances.
	\hidden{
		\begin{figure}
			\centering
			\includegraphics[scale=0.6]{./mallowsCardinal_3_6_Dispersion.png}
			\caption{Boxplot of the optimal value of ($k$-app envy)-free allocations with 3 agents and 6 objects VMF distributions under different dispersions}
			\label{fig:MallowsCardinalConc}
		\end{figure}
	}
	\paragraph{House allocation}
	As the polynomial algorithm runs much faster than our MIP, we were able to test the influence of the dispersion value on 100 instances of size $n=m=30$ for each concentration value from $0$ to $100$ by steps of $10$, then from $100$ to $1000$ by steps of $100$ and finally infinite concentration (meaning all the agents have the same preferences).
	In a very similar way we can see how Figure \ref{fig:MallowsOrdinal} confirms our intuition as we move from complete existence for low concentration values to total non-existence of our notion (all the instances are unanimous envy ones) when agents have the same preferences.
	\begin{figure}
		\centering
		\includegraphics[scale=0.6]{graph/mallowsPoly30_30_100_Indifferences_NoOutliers.png}
		\caption{Boxplot of the optimal ($K$-app envy)-free allocations level $K$ in the HAP with VMF distributions under different concentrations}
		\label{fig:MallowsOrdinal}
	\end{figure}
}

\section{\uppercase{Conclusion}}

In this paper, we have introduced a new relaxation of envy-freeness. This relaxation uses a consensus notion, approval envy, as a proxy for objective envy between pairs of agents. We have proposed algorithms to compute an allocation minimizing the $K$-app envy, and we have experimentally shown that this notion makes sense in practice in situations where no envy-free allocation exists.

This work also opens to a more general study of consensus-based notions of envy. For instance, instead of focusing on approval envy between agents, one could also be interested in using consensus to determine whether a given agent should be envious in general or not. More generally, one could also look for allocations that are judged envy-free by a given quota of agents. We leave the study of these notions for future work.


\hidden{
	\begin{example}
		\label{basic-example}
		Let us consider the following instance with 5 agents and 5 objects and the circled allocation $\pi$:
		\begin{center}
			$
			\begin{array}{c|cccc}
			& o_1 & o_2 & o_3 & o_4 \\
			\hline
			a_1 & 1 & 2 & \al{3} & 4   \\
			a_2 & 3 & \al{4} & 1 & 2    \\
			a_3 & \al{3} & 2 & 4 & 1    \\
			a_4 & 3 & 2 & 1 & \al{4}   \\
			\end{array}
			$
		\end{center}
		Agent $a_1$ is envious of agent $a_4$. Beyond $a_1$, two other agents approve this envy ( $a_2$ and $a_4$). Agent $a_3$ is also envious (of agent $a_1$), which is only approved by agent $a_1$ herself. 
	\end{example}
	
	Another version could be conceived where, instead of asking agents to vote on pairwise envies, they would vote on the fact that they believe the agent is envious or not. The following example illustrates the difference.
	
	\begin{example}
		\label{basic-example-ctd}
		Consider the same instance as in Example \ref{basic-example}, and the allocation such that agent $a_i$ holds $o_i$. 
		Agent $a_1$ is envious of all other agents. Each other agent $a_j$ agrees on the fact that $a_1$ should envy herself ($a_j$), but not the other ones. Hence, the allocation is (3-app envy)-free, however all agents agree that $a_1$ is right to be envious (but not for the same reasons).  
	\end{example}
}

\hidden{With the deliberative scenario sketched in the introduction in mind, we believe our pairwise definition is more appropriate. Indeed, it allows to identify exactly where disagreement occurs, hence leading to consensus after deliberation. In our example, it is not clear that agents $a_2$, $a_3$ and $a_4$ could convince $a_1$ she should not be envious if they do not agree in the first place.  }

\bibliography{partage.bib}
\end{document}